\PassOptionsToPackage{cleveref}{jmlrutils}
\PassOptionsToPackage{capitalize}{cleveref}
\documentclass{neus2025}
\jmlryear{2026}
\jmlrproceedings{}{Preprint}
\jmlrvolume{}
\jmlrpages{}
\jmlrworkshop{Accepted at NeuS 2026}

\title[Bridging Learned Visual Perception and Symbolic Belief-Space Planning]{Bridging Learned Visual Perception and Symbolic Belief-Space Planning via Probabilistic Grounding}
\usepackage{times}

\author{%
 \Name{Guy Azran} \Email{guy.azran@campus.technion.ac.il}\\
 \addr Taub Faculty of Computer Science, Technion - Israel Institute of Technology
 \AND
 \Name{Michael Navat} \Email{michaelnavat@campus.technion.ac.il}\\
 \addr Faculty of Mathematics, Technion - Israel Institute of Technology
 \AND
 \Name{Sarah Keren} \Email{sarahk@cs.technion.ac.il}\\
 \addr Taub Faculty of Computer Science, Technion - Israel Institute of Technology
}

\usepackage[nohyperlinks]{acronym}
\makeatletter
\renewcommand*\AC@verridelabel[1]{%
  \@bsphack
  \protected@write\@auxout{}{\string\AC@undonewlabel{#1}}%
  \protected@write\@auxout{}{\string\AC@undonewlabel{#1@cref}}%
  \label{#1}%
  \AC@overriddenmessage rs{#1}%
  \AC@overriddenmessage rs{#1@cref}%
  \@esphack
}
\makeatother

\usepackage{bbm}  %
\usepackage{comment}
\usepackage{algorithm}
\usepackage{algorithmic}
\usepackage{wrapfig}
\usepackage{caption}
\usepackage{booktabs}

\hypersetup{
  pdftitle={Bridging Learned Visual Perception and Symbolic Belief-Space Planning via Probabilistic Grounding},
  pdfauthor={Guy Azran, Michael Navat, Sarah Keren}
}

\newtheorem{claim}{Claim}

\DeclareMathOperator*{\argmax}{arg\,max}

\usepackage{xspace}

\newcommand{\tuple}[1]{\ensuremath{\left\langle #1 \right\rangle}\xspace}

\newcommand{\vlmFn}{\ensuremath{\phi}\xspace}
\newcommand{\token}{\ensuremath{v}\xspace}
\newcommand{\vocabSet}{\ensuremath{\mathcal{V}}\xspace}
\newcommand{\prompt}{\ensuremath{x}\xspace}
\newcommand{\calibError}{\ensuremath{\delta}\xspace}
\newcommand{\minPredImp}{\ensuremath{\gamma}\xspace}
\newcommand{\minVisRate}{\ensuremath{\rho}\xspace}
\newcommand{\finiteStep}{\ensuremath{T}\xspace}
\newcommand{\weightProb}{\ensuremath{W}\xspace}
\newcommand{\error}{\ensuremath{\varepsilon}\xspace}

\newcommand{\logOdds}{\ensuremath{L}\xspace}
\newcommand{\logObs}{\ensuremath{l}\xspace}

\newcommand{\fluentSet}[0]{\ensuremath{F}\xspace}
\newcommand{\initState}[0]{\ensuremath{I}\xspace}
\newcommand{\opSet}[0]{\ensuremath{A}\xspace}
\newcommand{\goalState}[0]{\ensuremath{G}\xspace}
\newcommand{\stripsTup}[0]{\tuple{\fluentSet, \initState, \opSet, \goalState}}

\newcommand{\fluent}[0]{\ensuremath{f}\xspace}
\newcommand{\op}[0]{\ensuremath{a}\xspace}

\newcommand{\beliefSet}{\ensuremath{b}\xspace}
\newcommand{\initBeliefSet}[0]{\ensuremath{\beliefSet^\initState}\xspace}
\newcommand{\beliefState}[0]{\ensuremath{\beta}\xspace}
\newcommand{\initBeliefState}[0]{\ensuremath{\beta^\initState}\xspace}

\newcommand{\factoredBelief}[0]{\ensuremath{\beliefState_\fluentSet}\xspace}

\newcommand{\probThresh}{\ensuremath{\theta}\xspace}
\newcommand{\sensorModel}{\ensuremath{M}\xspace}

\newcommand{\prob}{\ensuremath{p}\xspace}
\newcommand{\constraintSet}{\ensuremath{\mathcal{C}}\xspace}
\newcommand{\normalizer}{\ensuremath{Z}\xspace}
\newcommand{\unconstBelief}{\ensuremath{\tilde{\beliefState}}\xspace}

\newcommand{\policy}{\ensuremath{\pi}\xspace}

\newcommand{\obsSpace}[0]{\ensuremath{O}\xspace}
\newcommand{\obs}[0]{\ensuremath{o}\xspace}
\newcommand{\stateSpace}[0]{\ensuremath{S}\xspace}

\newcommand{\state}[0]{\ensuremath{s}\xspace}

\newcommand{\seFn}[0]{\ensuremath{\sigma}\xspace}

\newcommand{\actionMapper}[0]{\ensuremath{\xi}\xspace}
\newcommand{\workspace}[0]{\ensuremath{W}\xspace}
\newcommand{\wsConfig}[0]{\ensuremath{w}\xspace}

\newcommand{\beliefSize}[0]{\ensuremath{k}\xspace}

\newcommand{\queue}[0]{\ensuremath{Q}\xspace}
\newcommand{\queueExtracted}[0]{\ensuremath{\stateSpace^{\text{ext}}}\xspace}

\newcommand{\plan}[0]{\ensuremath{\Pi}\xspace}

\newcommand{\indicator}[0]{\ensuremath{\mathbbm{1}}\xspace}

\usepackage{listings}
\begin{document}

\maketitle

\begin{abstract}
  In partially observable settings, agents must act without full knowledge of the world state and rely on uncertain state-estimation pipelines.
Obtaining grounded and verifiable symbolic plans under such uncertainty remains a key challenge.
Recent work has integrated \acp{vlm} to bridge perception and symbolic reasoning, following two main paradigms. The first, \ac{vlm}-as-planner, maps images directly to action sequences, and the second, \ac{vlm}-as-grounder, grounds observations into symbolic predicates used as the initial state by off-the-shelf planners. Both approaches ignore uncertainty in the planning process, compromising robustness.
We introduce a third paradigm, \ac{vlm}-as-probabilistic-grounder, a novel approach that captures the uncertainty of \ac{vlm} predicate groundings as a probability distribution over symbolic states.
This enables planning in belief space and producing robust plans under uncertainty.
Experiments in simulated household robot settings show improved robustness and task success over deterministic grounding, underscoring how our approach leverages foundation models for reliable planning under uncertainty.

\end{abstract}


\acresetall

\section{Introduction}

\begin{wrapfigure}{r}{0.5\linewidth}
    \centering
    \subfigure[``Book to shelf'']{
        \label{fig:example:radio}
        \includegraphics[width=0.45\linewidth]{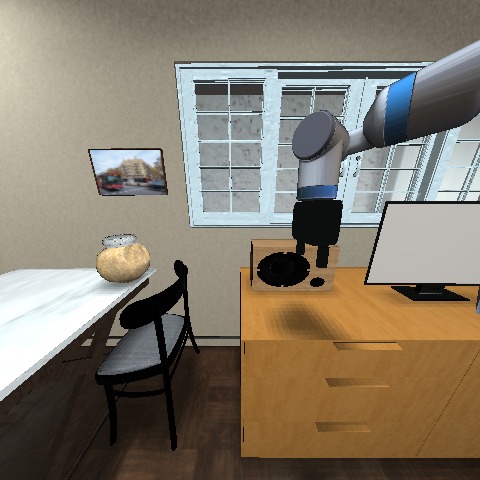}
    }\hfill
    \subfigure[``Bowl to sink'']{
        \label{fig:example:kitchen}
        \includegraphics[width=0.45\linewidth]{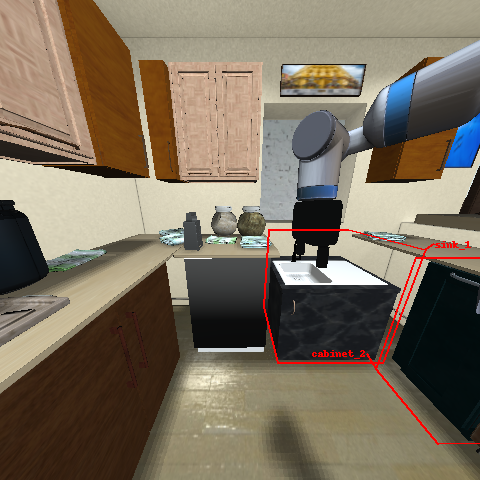}
    }
    \caption{Observations for two tasks in the ViPlan household benchmark \citep{merlerViPlanBenchmarkVisual2025}.}
    \label{fig:example}
    \vspace{-1\baselineskip}
\end{wrapfigure}

Robotic applications remain challenging due to the inherent uncertainty in perception and action outcomes.
Many traditional frameworks adopt the \acl{cwa} \citep{reiterCLOSEDWORLDDATA1981}, treating facts not known to be true as false, and rely on handcrafted, task-specific solutions \citep{wertheimPlugnPlayTaskLevel2024,morenoCombinedTaskMotion2024,ranaSayPlanGroundingLarge2023,garrettPDDLStreamIntegratingSymbolic2020}.
Recently, \acp{vlm} have been incorporated to infer task-relevant information from visual observations, enabling more generalizable and flexible planning systems \citep{zhangDKPROMPTDomainKnowledge2024,huLookYouLeap2023}. \citet{merlerViPlanBenchmarkVisual2025} distinguish between two paradigms for \ac{vlm}-planning integration: \ac{vlm}-as-planner, where the \ac{vlm} directly generates plans from visual inputs, and \ac{vlm}-as-grounder, where the \ac{vlm} provides symbolic predicates that an off-the-shelf planner uses to compute plans.
These approaches disregard the inherent uncertainty in visual grounding and treat \ac{vlm} predictions as deterministic.
This is ineffective in domains such as household robotics, where perception is noisy and ambiguous, and producing deterministic plans may fail when ambiguity and missing information lead to incorrect predicate assignments.

\cref{fig:example} shows two tasks that demonstrate planning uncertainty due to misleading \acs{vlm} predictions and partial observability.
In \cref{fig:example:radio}, the robot must place a book on the target shelf, which is outside the frame. A radio appears behind the robot's gripper, but the \ac{vlm} believes it is holding the book. The robot plans to navigate to the shelf, where it will realize its mistake, triggering a replan.
In \cref{fig:example:kitchen}, the robot must bring a hidden bowl to the sink. The \ac{vlm} has no indication that the bowl is in a cabinet, so it will repeatedly plan to navigate to the bowl that it cannot see and never complete the task.
The robot can overcome these challenges by maintaining a belief over possible world states (e.g., book in hand and not in hand, or bowl in cabinet or not in cabinet) and planning to solve the task accordingly.
This will enable the robot to solve the task with less replanning.

In line with this, we present \ac{sc}, which leverages a \ac{vlm} to guide a robust decision-making process rather than blindly trusting its output.
We use the \ac{vlm} to produce predicate probabilities to maintain an explicit belief over high-level states and frame the problem as \ac{cpp} \citep{domshlakFastProbabilisticPlanning2006}.
This produces more robust plans while retaining the interpretability improvements of symbolic grounding.
The contributions of our work are as follows:\vspace{-7pt}
\begin{enumerate}
\setlength{\itemsep}{-2pt}
    \item We introduce \ac{vlm}-as-probabilistic-grounder, a new \ac{vlm}-planning paradigm leveraging fluent-level probability instead of brittle deterministic grounding or direct action generation.
    \item We formalize the \ac{rvp} problem, providing a principled definition of planning under partial observability and perceptual uncertainty in visual robotic domains.
    \item We propose \ac{sc}, a robust visual task planner that maintains a symbolic belief derived from \ac{vlm}-based probabilities and compiles it into a \acf{cp} problem.
    \item We develop a theoretically grounded planning-execution loop with guarantees on correctness and safety, including conditions under which the \ac{vlm} is sufficiently accurate and useful.
\end{enumerate}\vspace{-10pt}

We evaluate our approach on the ViPlan-HH benchmark \citep{merlerViPlanBenchmarkVisual2025}, which contains multiple robot planning tasks in various home scenes. Our experiments demonstrate how \ac{sc}'s robust plans enable robots to solve complex tasks under uncertainty where other approaches fail.

\begin{figure}[t]
\begin{minipage}[t][][t]{0.42\textwidth}
    \centering
\fbox{\includegraphics[width=1\linewidth]{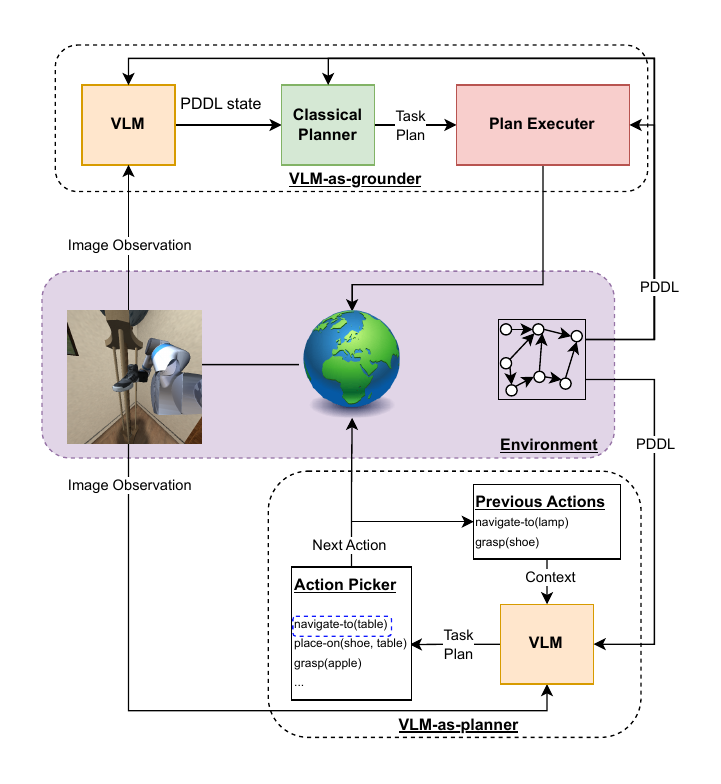}}
    \caption{ \ac{vlm}-as-grounder (top) and \ac{vlm}-as-planner (bottom)}
    \label{fig:methods:theirs}
\end{minipage}\hfill
\begin{minipage}[t][][t]{0.53\textwidth}
    \centering
    \fbox{\includegraphics[width=1\linewidth]{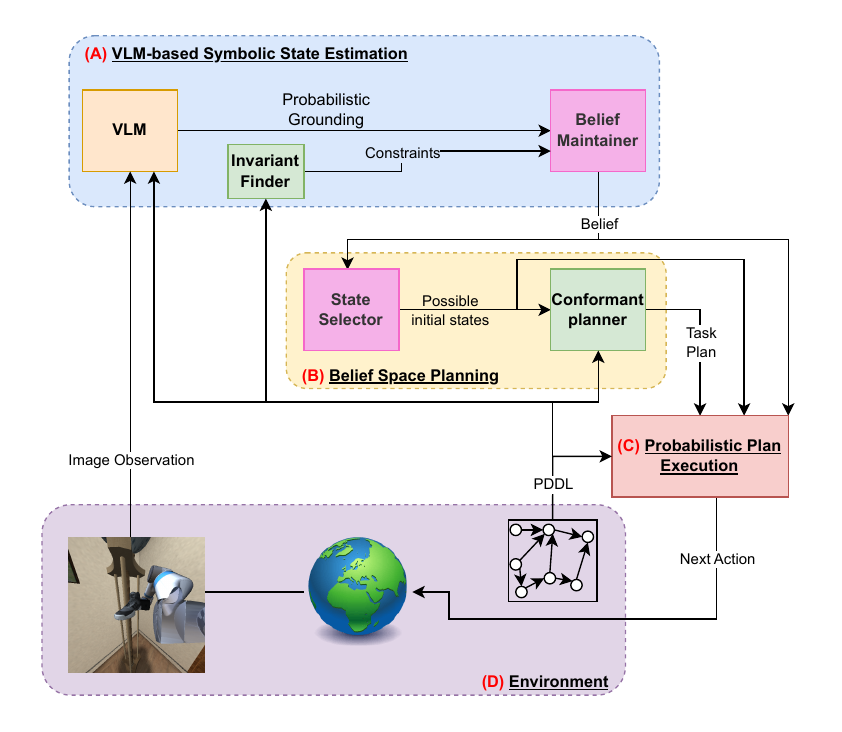}}
    \caption{\acf{sc} pipeline}
    \label{fig:methods:ours}
\end{minipage}
\end{figure}

\acresetall

\vspace{-7pt}

\section{Background and Related Work}

Our approach is based on harnessing the power of pre-trained \acp{vlm} to answer semantic queries about visual input \citep{liGroundedLanguageImagePretraining2022,radfordLearningTransferableVisual2021,liVisualBERTSimplePerformant2019}.
Let \obsSpace be a set of possible image observations, and let \vocabSet be a set of textual tokens called the vocabulary.
A \ac{vlm} is a function $\vlmFn: \obsSpace \times \vocabSet^* \to [0, 1]^{\vocabSet}$ that takes as input an image observation $\obs \in \obsSpace$ and a sequence of textual tokens (a prompt) $\prompt = \tuple{\token_1, \ldots, \token_n} \in \vocabSet^*$. Its output is a prediction for the next token in the sequence.
By iteratively inserting the predicted token into the prompt, the \ac{vlm} generates textual outputs conditioned on visual inputs.

Recent work in \ac{vla} robotics uses visual and language inputs either to produce low-level robot actions directly or to support task-level planning. End-to-end \ac{vla} policies map observations and natural language instructions to robot controls \citep{kimOpenVLAOpenSourceVisionLanguageAction2024,duanManipulateAnythingAutomatingRealWorld2024,jiangVIMAGeneralRobot2023}, while planning-loop approaches use \acp{vlm} as high-level planners or symbolic grounders \citep{ranaSayPlanGroundingLarge2023,zhangDKPROMPTDomainKnowledge2024,huLookYouLeap2023}. In this work, we focus on the latter.

\citet{merlerViPlanBenchmarkVisual2025} distinguish two paradigms for integrating a \ac{vlm} into a planning loop.
In \ac{vlm}-as-planner, depicted in \cref{fig:methods:theirs} (bottom), the model maps an image observation and task description directly to an action sequence \citep{yangGuidingLongHorizonTask2025,duanManipulateAnythingAutomatingRealWorld2024,huLookYouLeap2023}.
In \ac{vlm}-as-grounder, depicted in \cref{fig:methods:theirs} (top), the model assigns Boolean values to grounded task fluents, and a symbolic planner plans from the resulting state estimate \citep{azranS3ESemanticSymbolic2025,liangVisualPredicatorLearningAbstract2024,chenLLMStateOpenWorld2024,dingRobotTaskPlanning2024,zhangDKPROMPTDomainKnowledge2024}.
Both paradigms typically commit to a single \ac{vlm} output rather than propagating uncertainty over multiple plausible symbolic states.
Our approach instead extracts per-fluent probabilities, constructs a belief over symbolic states, selects states whose cumulative probability is at least \probThresh, and computes a conformant plan that is valid for every selected state.

These task-level approaches typically use a deterministic symbolic action model \citep{ghallabAutomatedPlanningTheory2004}. Grounder-based systems pass one grounded state to a classical planner, while planner-based systems generate an action sequence expressed in the same symbolic model.
Such models are commonly represented using the \acs{strips} formalism \citep{fikesStripsNewApproach1971}. \acs{strips} defines a planning problem as a tuple \stripsTup, where \fluentSet is a set of fluents representing the state of the world, \initState is the initial state, \opSet is a set of actions (operators) that have preconditions that determine when they are applicable and can change the state via their effects, and $\goalState \subseteq \fluentSet$ is the goal condition.
While it is common to use classical planning algorithms and replanning upon unexpected outcomes \citep{yoonFFReplanBaselineProbabilistic2007}, this is highly ineffective in settings in which replanning is costly or even impossible.
For example, a robot may need to communicate with an external computation source to plan, taking a long time due to bad connectivity or even failing when there is no connectivity.

On the other end of the spectrum, \emph{\ac{cp}} \citep{palaciosCompilingUncertaintyAway2009,smithConformantGraphplan1998} addresses planning under partial observability by generating plans that are guaranteed to achieve the goal from any possible initial state.
In \ac{cp}, the agent must find a plan that will achieve the goal from a set of possible states called the \emph{belief set}, denoted \initBeliefSet.
\emph{\Ac{cpp}} \citep{domshlakFastProbabilisticPlanning2006} extends this framework by replacing the belief set with a probability distribution over possible states $\beliefState: 2^\fluentSet \rightarrow [0,1]$, called a \emph{belief state}, or \emph{belief}.
The objective of \ac{cpp} is to find a plan that achieves the goal with some threshold probability \probThresh. 
Our work builds upon \ac{cpp} by utilizing \acp{vlm} to define and update the belief based on visual observations, allowing for more informed planning under uncertainty in visually rich environments. We find robust plans using insights from \citet{taigCompilingConformantProbabilistic2013} and using off-the-shelf conformant planners \citep{maliahComputingContingentPlan2022,shaniReplanningDomainsPartial2011,palaciosCompilingUncertaintyAway2009}.

\vspace{-7pt}

\section{Problem Formulation}

We aim to construct perception–action pipelines that ground symbolic reasoning from raw visual input. These pipelines should produce satisficing task-level plans that remain effective despite sensor noise, occlusion, and partial observability. We therefore formulate \emph{\ac{rvp}} for agents operating in complex, partially observable environments. Hereafter, we focus on an embodied robotic agent, and therefore refer to it as a robot.
In \ac{rvp}, a robot must perceive the world through onboard sensors (e.g., cameras, depth, proprioception), infer symbolic state information from raw visual input, and select high-level actions to execute. These are represented as \emph{skills}, i.e., action implementations for the robot realized by low-level motion controllers.

We formalize the problem as a tuple
$\tuple{\fluentSet, \opSet, \goalState, \workspace, \obsSpace, \sensorModel, \actionMapper}$,
which augments the symbolic task model with a continuous configuration space, an observation model, and stochastic skill execution.
Each component captures a layer in the perception-action hierarchy of a robotic agent:\vspace{-7pt}
\begin{itemize}
\setlength{\itemsep}{-2pt}
    \item $\fluentSet$, $\opSet$, and $\goalState$ define the \textbf{symbolic task space}, following \acs{strips} \citep{fikesStripsNewApproach1971}.
    \item $\workspace$ represents the \textbf{robotic workspace}, comprising the set of all feasible continuous configurations of both the robot and the manipulable objects in the scene.
    A workspace configuration $\wsConfig \in \workspace$ encodes the robot's joint positions, object poses, and environmental states.
    \item $\obsSpace$ denotes the \textbf{observation space} available to the robot, such as RGB-D or multi-view camera images.
    \item $\sensorModel: \workspace \to \Delta\obsSpace$ is the \textbf{observation model}. Given a workspace configuration $\wsConfig$, the robot receives an observation $\obs \sim \sensorModel(\wsConfig)$ that reflects sensor noise, occlusion, and limited field of view.
    \item $\actionMapper: \opSet \times \workspace \to \Delta\workspace$ is the skill executor, which translates symbolic actions to low-level control policies that are executable in the workspace.
    Invoking $\actionMapper(\op, \wsConfig)$ returns the stochastic result of applying a motion primitive or controller (e.g., for grasping or navigation) that attempts to realize the symbolic effects of $\op \in \opSet$ on current workspace configuration $\wsConfig \in \workspace$.
\end{itemize}\vspace{-7pt}
Note that the high-level STRIPS model is a deterministic view of the world, which defines how we expect actions to affect the workspace. However, the actual transition is handled by the non-deterministic skill executor \actionMapper.

Our objective is to harness the capabilities of \acp{vlm} and equip robotic agents with robust planning capabilities in visually complex and uncertain environments.
Assuming that for evaluation, we have access to ground-truth mapping $\seFn: \workspace \to 2^\fluentSet$ from workspace configurations to high-level \acs{strips} states, and to the initial workspace configuration $\wsConfig_0$, we aim to generate policies that maximize the probability of reaching a goal-satisfying workspace configuration under partial observability.

\section{\acf{sc}}

Our approach, \emph{\acf{sc}}, integrates \ac{vlm}-based symbolic state estimation, probabilistic planning, and execution strategies for robust decision-making under perceptual uncertainty in \ac{rvp}.

The \ac{sc} pipeline is depicted in \cref{fig:methods:ours}. Component (A) translates an image observation into a symbolic belief by using a \ac{vlm} to obtain a probability distribution over task-relevant facts, maintained via logarithmic opinion-pooling updates. Component (B) takes the belief and produces a robust plan, selecting a subset of possible states whose cumulative probability surpasses a user-specified threshold and invoking a conformant planner that satisfies all states in the subset. Component (C) handles execution, selecting actions and triggering replanning.

\subsection{\acs*{vlm}-Based Symbolic State Estimation}
In the first stage of the pipeline, we infer a high-level representation of the state from an input image observation, as depicted in component (A) of \cref{fig:methods:ours}.

The common approach to using \acp{vlm} for symbolic state estimation involves prompting the model with visual observations (e.g., images from the robot's cameras) and querying it for the truth values of task-relevant predicates in a deterministic fashion \citep{azranS3ESemanticSymbolic2025,merlerViPlanBenchmarkVisual2025,zhangDKPROMPTDomainKnowledge2024}. We also prompt the model for true-false labels of specific fluents \fluent, but extract the internal probability that the \ac{vlm} assigns to the next-token prediction $\vlmFn(\obs, \prompt)$ over the vocabulary, and normalize the true-false distribution to define a belief over the truth value of $\fluent$.

Intuitively, we would expect the \ac{vlm} to assign ``uncertain'' probabilities to true and false (e.g., both around 0.5) in cases where the fluent is unobservable. However, \acp{vlm} are typically trained on data where the relevant information is already in the image. Therefore, questions about objects not present in the image are considered out-of-distribution, and thus, we cannot expect them to be calibrated for this.
To account for unobservable or uncertain fluents, we allow a third response label \texttt{null} for cases where there is not enough evidence to confidently assign true or false.

We define \emph{probabilistic symbolic grounding} as the extraction of a probability for each $\fluent \in \fluentSet$.
Denote by $\prompt_\fluent$ the prompt used to query fluent \fluent, e.g., ``Is the cabinet currently open?'' for fluent \verb|open(cabinet)|.
\citet{azranS3ESemanticSymbolic2025} showed how to generate these prompts autonomously given a symbolic representation of the domain, e.g., using \ac{pddl}.

To obtain a probability $\prob_{\obs,\fluent}$ for each fluent $\fluent \in \fluentSet$ for that specific observation \obs and \ac{vlm} \vlmFn, we calculate the probability of fluent \fluent. First, we extract the \ac{vlm} output probabilities for the tokens ``true'', ``false'', and ``null''. If the ``null'' token is most likely, we set the probability for that fluent to $\prob_{\obs,\fluent} = 0.5$. Otherwise, we normalize the true-false probabilities $\prob_{\obs,\fluent} = \frac{\prob^{true}_{\obs,\fluent}}{\prob^{true}_{\obs,\fluent} + \prob^{false}_{\obs,\fluent}}$. See \cref{ap:probground} for full details on the probabilistic grounding computation.

We define a \emph{factored belief} $\factoredBelief = (\beliefState_\fluent)_{\fluent\in\fluentSet}\in [0,1]^\fluentSet$ as a belief maintained over individual fluents. Using this representation, we perform per-fluent logarithmic opinion-pooling updates \citep{neymanNoRegretLearningUnbounded2023,genestCombiningProbabilityDistributions1986}. We calculate $\beliefState_\fluent \gets \mathrm{expit}(\mathrm{logit}(\beliefState_\fluent) + \mathrm{logit}(\prob_{\obs,\fluent}))$.
Note that if $\prob_{\obs,\fluent} = 0.5$, then $\mathrm{logit}(\prob_{\obs,\fluent}) = 0$, and thus the belief about fluent \fluent remains unchanged, as expected for an uninformative observation.

To obtain a belief over the full symbolic state, we assume conditional independence between fluents, subject to constraints \constraintSet, e.g., mutual exclusion and co-dependence defined by the task domain.
This factorization is an approximation. The constraints encode known dependencies, while unmodeled correlations may enlarge the \ac{mlss} and make planning conservative.
States that violate \constraintSet are assigned zero probability, and the rest are normalized accordingly.

Denote $\state(\fluent) = \indicator_{\fluent \in \state}$. The belief is calculated as follows:

\begin{minipage}{0.48\textwidth}
\begin{equation}
    \unconstBelief(\state) = \prod_{\fluent \in \fluentSet} \beliefState_\fluent^{\state(\fluent)} (1 - \beliefState_\fluent)^{1 - \state(\fluent)} \label{eq:calc-unconst-belief}
\end{equation}
\end{minipage}
\hfill
\begin{minipage}{0.48\textwidth}
\begin{equation}
    \beliefState(\state) = \begin{cases}
    \frac{\unconstBelief(\state)}{\normalizer} & \text{if } \state \text{ satisfies } \constraintSet \\
    0 & \text{otherwise}
    \end{cases} \label{eq:calc-belief}
\end{equation}
\end{minipage}
where $\normalizer \leq 1$ is a normalization constant.
We call $\unconstBelief$ the \emph{unconstrained belief}, as it does not take into account the constraints in \constraintSet.
Note that for \beliefSet to be well-defined, we must assume there exists at least one state \state that satisfies \constraintSet with a non-zero $\unconstBelief(\state)$.

To extract the constraint set \constraintSet, we use the Fast Downward planning system's finite domain representation \citep{helmertFastDownwardPlanning2006}. This reveals several constraints, e.g., mutually exclusive fluent groups, exposed in a dedicated invariant discovery stage. We can use the graphs generated by this invariant finder to calculate \normalizer in polynomial time. However, ignoring this term, we can view our belief as a lower bound on the actual belief.  This will be enough to obtain a state space that guarantees a desired threshold probability of success in the underlying \ac{cpp} problem (see next section).

\vspace{-7pt}
\subsection{Conformant Probabilistic Planning with \acs*{vlm}-based Belief}
Next in the pipeline is component (B) in \cref{fig:methods:ours}, which generates a robust plan that satisfies the goal with high probability according to our maintained belief.

The \ac{cp} paradigm addresses planning under initial state uncertainty. \ac{rvp} raises two sources of uncertainty, namely perceptual uncertainty when translating from observations to symbolic states, and partial observability in a single observation that may not fully reveal the true symbolic state. When the agent receives an observation, we want to define a \ac{cp} problem that reflects the current belief.

Let \tuple{\fluentSet, \beliefState, \opSet, \goalState, \probThresh} be a \ac{cpp} problem.
\citet{taigCompilingConformantProbabilistic2013} prove that solving this problem is equivalent to solving a \ac{cp} problem \tuple{\fluentSet, \initBeliefSet, \opSet, \goalState} such that the probability that the initial state is in \initBeliefSet is greater than \probThresh, i.e., $\sum_{\state \in \initBeliefSet} \beliefState(\state) \geq \probThresh$.
However, their approach used a planner to find a \initBeliefSet with the cheapest solution. 
We want the robot to act on the best interpretation of the visual scene, and so we propose finding the most likely subset of states that satisfy the probability threshold.

\begin{definition}
    Let \factoredBelief be a factored belief, $\probThresh \in [0,1]$ be a probability threshold, and \constraintSet be a set of constraints over \fluentSet. The \acf{mlss}, denoted by $MLSS(\factoredBelief, \probThresh, \constraintSet)$, is a subset of states $\initBeliefSet \subseteq 2^\fluentSet$ of minimal cardinality such that $\sum_{\state \in \initBeliefSet} \beliefState(\state) \geq \probThresh$.
\end{definition}

\cref{alg:mlss} describes our approach to finding an \ac{mlss} by performing a search over boolean fluent value flips.
In lines \ref{alg:mlss:init-s0}-\ref{alg:mlss:init-sets}, we initialize the search with the most likely state $\state_0$ and set up a max-heap containing only this state.
In line \ref{alg:mlss:extract-q}, we extract the most likely state from the heap, based on its unconstrained belief \unconstBelief, i.e., the belief before applying constraints and normalizing (computed via \cref{eq:calc-unconst-belief}).
The extracted state is added to the output belief set if it satisfies the constraints (line \ref{alg:mlss:append-res}).
The loop starting at line \ref{alg:mlss:for} explores neighboring states by flipping each fluent in turn, adding unvisited neighbors to the heap (line \ref{alg:mlss:insert-q}).
The termination condition (line \ref{alg:mlss:while}) checks whether the cumulative probability of the selected states meets or exceeds \probThresh.
The key insight behind this algorithm is that although states are extracted from the heap according to the unconstrained belief \unconstBelief, they are guaranteed to be in non-increasing order of the belief \beliefState.

\noindent
\begin{minipage}[t][][t]{0.46\textwidth}
        \centering
    \captionof{algorithm}{\acl{mlss}}
\label{alg:mlss}
    \hrule
\begin{algorithmic}[1]
\REQUIRE $\factoredBelief: \fluentSet \rightarrow [0,1]$, $\probThresh \in [0, 1]$, \constraintSet -- set of constraints.
\ENSURE $MLSS(\factoredBelief, \probThresh, \constraintSet)$
\STATE $\state_0 \leftarrow \indicator[\beliefState_\fluent > 0.5] \quad \forall \fluent \in \fluentSet$  \label{alg:mlss:init-s0}
\STATE $\queue \leftarrow$ Empty max-heap
\STATE $\queue$.Insert($\state_0$, $\unconstBelief(\state_0)$) $\quad$ \COMMENT{\unconstBelief from \cref{eq:calc-unconst-belief}}
\STATE $\text{Visited} \leftarrow \{\state_0\}$, $\beliefSet \leftarrow \emptyset$ \label{alg:mlss:init-sets}
\STATE $\prob \leftarrow 0$ \label{alg:mlss:init-p}
\WHILE{$\prob < \probThresh$ \AND $\queue \neq \emptyset$} \label{alg:mlss:while}
    \STATE $\state \leftarrow$ $\queue$.ExtractMax() \label{alg:mlss:extract-q}
    \IF{$\state \models \constraintSet$}
        \STATE $\beliefSet \leftarrow \beliefSet \cup \{\state\}$ \label{alg:mlss:append-res}
        \STATE $\prob \leftarrow \prob + \beliefState(\state)$ $\quad$ \COMMENT{\beliefState from \cref{eq:calc-belief}}
    \ENDIF
    \FOR{$\fluent \in \fluentSet$} \label{alg:mlss:for}
        \STATE $\state' \leftarrow$ Flip($\state$, $\fluent$)
        \IF{$\state' \notin \text{Visited}$}
            \STATE $\text{Visited} \leftarrow \text{Visited} \cup \{\state'\}$
            \STATE $\queue$.Insert($\state'$, $\unconstBelief(\state')$) \label{alg:mlss:insert-q}
        \ENDIF
    \ENDFOR
\ENDWHILE
\RETURN $\beliefSet$
\end{algorithmic}
    \hrule
\end{minipage}
\hfill %
\begin{minipage}[t][][t]{0.46\textwidth}
    \centering
    \captionof{algorithm}{\Ac{sc} Execution Loop}
\label{alg:replan}
    \hrule
\begin{algorithmic}[1]
\REQUIRE $\tuple{\fluentSet, \opSet, \goalState, \workspace, \obsSpace, \sensorModel, \actionMapper}$ is an \ac{rvp} problem, $\probThresh \in [0,1]$, $\initBeliefState: 2^\fluentSet \to [0,1]$, and $\obs \in \obsSpace$
\STATE $\constraintSet \leftarrow$ ExtractConstraints($\fluentSet, \opSet$)
\STATE $\factoredBelief \leftarrow$ VLM-belief($\obs, \vlmFn, \initBeliefState, \constraintSet$)
\WHILE{$\beliefState(\goalState) < \probThresh$}
    \STATE $\initBeliefSet \leftarrow$ MLSS($\factoredBelief, \probThresh, \constraintSet$)
    \STATE $\plan \leftarrow$ \ac{cp}($\fluentSet, \initBeliefSet, \opSet, \goalState, \probThresh$)
    \FOR{$\op \in \plan$}
        \IF{Unsafe($\op, \initBeliefSet$)}
            \STATE \textbf{break} $\quad$ \COMMENT{Replan}
        \ENDIF
        \STATE Execute($\op$, \actionMapper)
        \STATE $\obs \gets$ SensorReading()
        \STATE $\factoredBelief \leftarrow$ VLM-belief($\obs, \vlmFn, \factoredBelief, \constraintSet$)
        \IF{Improbable($\plan_{\text{remaining}}, \factoredBelief$)}
            \STATE \textbf{break} $\quad$ \COMMENT{Replan}
        \ENDIF
    \ENDFOR
\ENDWHILE
\end{algorithmic}
    \hrule
\end{minipage}

\begin{theorem}\label{thm:mlss}
    In \cref{alg:mlss}, when a constraint-satisfying state $\state$ is extracted from the heap $\queue$, then for all $\state'\notin \beliefSet$ it holds that $\beliefState(\state) \geq \beliefState(\state')$, where \beliefState is the belief computed by \cref{eq:calc-belief}.
\end{theorem} \vspace{-10pt}
The proof asserts that states are extracted from the heap in non-increasing order of their unconstrained belief \unconstBelief across the entire state space $2^\fluentSet$. Find the full proof in \cref{ap:proof-mlss}.

In the worst case, the number of explored states is exponential in $|\fluentSet|$.
Even if we bound the minimal subset size to \beliefSize, there might still be many constraint-violating states that must be explored before finding \beliefSize valid states.
In the absence of constraints, we can provide a tighter bound on the runtime, hoping that in practice we encounter a manageable number of constraint-violating states.

\vspace{-3pt}
\begin{proposition}\label{thm:efficient}
    If $\constraintSet = \emptyset$, \cref{alg:mlss} runs in $O(|\fluentSet|\beliefSize \log |\fluentSet|\beliefSize )$ time with $O(|\fluentSet|^2\beliefSize)$ bits, where $\beliefSize$ is the number of states in the \ac{mlss}.
\end{proposition}

\vspace{-5pt}
With this, we can efficiently generate the \ac{cpp} problem that arises from the \ac{vlm}-based belief state estimation. By characterizing the accuracy of a \ac{vlm} by its cumulative error over all fluents, we can guarantee that if the \ac{vlm}'s per-fluent predictions are \emph{sufficiently} accurate, the true state will be included in the \ac{mlss}. This enables the user to define bounds on sufficiency and accuracy. As a first step, we prove a sufficient condition on the factored belief that guarantees that the true state is included in the \ac{mlss}.

\begin{definition}[Cumulative factored belief error]\label{def:calib-err}
    Let $\state \in 2^\fluentSet$ be some high-level state of the world, and let $\factoredBelief = (\beliefState_\fluent)_{\fluent\in\fluentSet}$ be a factored belief. The \emph{cumulative factored belief error} is $
    \calibError(\factoredBelief, \state) =  \sum_{\fluent \in \fluentSet} |\beliefState_\fluent - \state(\fluent)|$.
\end{definition}

\begin{theorem}\label{thm:calib}
    Let $\probThresh \in (0,1]$, let \constraintSet be a set of constraints over \fluentSet, and let $\state \in 2^\fluentSet$ satisfy \constraintSet. For every $\factoredBelief \in [0,1]^\fluentSet$, if $\calibError(\factoredBelief, \state) < \probThresh$, then $\state \in MLSS(\factoredBelief, \probThresh, \constraintSet)$.
\end{theorem}

The proof uses the Weierstrass product inequality to lower-bound the probability of a constraint-satisfying state. The full proof is provided in \cref{ap:proof:calib}. If the true current state $\state^*$ is in the \ac{mlss}, any conformant plan found for the \ac{mlss} is valid from $\state^*$ under the symbolic model. We next give sufficient conditions under which repeated observations of a fixed state can reduce the cumulative factored belief error in a finite number of steps.

\begin{definition}[Weakly calibrated \ac{vlm}]
    For fluent probability $\prob_\fluent$ and true state fluent assignment $\state^*(\fluent)$, define $\prob_{\text{correct}}$ as the probability of the true assignment, i.e., $\prob_\fluent$ if $\state^*(\fluent) = 1$ and $1 - \prob_\fluent$ otherwise, with complement $\prob_{\text{incorrect}} = 1 - \prob_{\text{correct}}$.
    A \ac{vlm} is \emph{weakly calibrated} if there exists $\minPredImp > 0$ such that for every observation $\obs\in\obsSpace$ and fluent $\fluent \in \fluentSet$, the probability $\prob_\fluent$ derived for fluent \fluent is at least \minPredImp closer to the correct prediction than the incorrect one, i.e., $\prob_{\text{correct}} - \prob_{\text{incorrect}} \geq \minPredImp$.
\end{definition}

\begin{definition}[Minimum visibility rate]
    The \emph{minimum visibility rate} $\minVisRate$ is the frequency with which the rarest fluent is observed, where the rarest fluent is a maximizer in $
    \argmax_{\fluent\in \fluentSet} \Pr(\prob_{\obs,\fluent}^{null} > \max\{\prob_{\obs,\fluent}^{true}, \prob_{\obs,\fluent} ^{false}\})$
\end{definition}

\begin{corollary}\label{thm:fin-step}
    Let $\minVisRate > 0$ be the minimum visibility rate. If the \ac{vlm} is weakly calibrated with error margin $\minPredImp$, then there exists a finite number of steps $\finiteStep$ after which a conformant plan for the \ac{mlss} according to \factoredBelief is satisficing from the true current state.
\end{corollary}

In the proof, we find a lower bound on the current log odds in the belief that grows over time, and push it above $\probThresh$. Thus, the \ac{vlm} need not be perfect. It only needs to be persistently better than a random guess on visible fluents to eventually generate a satisficing plan for the real-world state. The full proof is in \cref{ap:proof:fin-step}.

These assumptions describe when repeated visible predictions provide consistent evidence for each fluent. They are weaker than assuming error-free deterministic grounding.
Both assumptions can be empirically evaluated for a given \ac{vlm} and task domain. The parameters \minPredImp and \minVisRate can then be used to calculate the number of steps required by \cref{thm:fin-step}.

Of course, when using a conformant planner, there is an inherent tradeoff between robustness and completeness. The \ac{cpp} component of \ac{sc} allows the user to iteratively adjust this by calibrating the success probability threshold \probThresh if a plan is not found.

\vspace{-7pt}
\subsection{Executing Conformant Probabilistic Plans}
We aim to support robotic settings, wherein executing a plan may lead to unexpected outcomes due to perceptual uncertainty, partial observability, and action failure.
As such, part of the robustness of our policy relies on the ability of the executor to detect failure and inconsistency. A replan may be triggered when new observations indicate that the current belief is inconsistent with the belief set used for planning, or the action execution failed. In the classical case, initiating replanning is straightforward: when an action cannot be executed, or its observed outcome does not match the expected state, the agent replans from the new observed state \citep{yoonFFReplanBaselineProbabilistic2007}. In the conformant case, there are a few more considerations.

To handle replanning, we introduce a novel conformant plan executor, handled in component (C) of \cref{fig:methods:ours}.
We need to determine when the current plan no longer aligns with the belief derived from new observations.
Our execution component introduces a belief-consistency-based monitoring criterion, a probabilistic monitoring scheme tailored for \ac{vlm} uncertainty.
\acp{vlm} can drastically change their predictions from one observation to another, so the standard ``expected state vs observed state'' criterion can lead to constant replanning.
As such, we propose three replanning triggers based on this misalignment, namely improbable plan, unsafe action, and plan exhaustion.
We focus on belief drift via the improbable plan trigger, which requires the \ac{vlm} to output predictions that are persistently inconsistent with the belief and with great confidence.
The unsafe action trigger protects against ``hallucinated certainty'', where a \ac{vlm} might initially be confident but wavers as the robot approaches.
Finally, the plan exhaustion trigger ensures that if the plan is not successful after execution, we replan to correct for any misinterpretations of the environment.
To our knowledge, this is the first execution framework that explicitly reasons about \ac{vlm}-induced belief drift rather than single-step perceptual inconsistency.
The full execution loop is outlined in \cref{alg:replan}.

An improbable plan is one whose probability of success is below a threshold \probThresh given the current belief.
As a proxy for this value, we calculate the probability of the current belief set \initBeliefSet, after executing all actions taken thus far.
If this value drops below \probThresh, it is possible that the plan's success probability is also below \probThresh, and so we replan.

The resulting policy from \cref{alg:replan}, which we denote as $\policy_{cpp}$, provides a conditional robustness guarantee based on the accuracy of our perceptual belief, according to \cref{thm:calib}. This execution strategy also provides a safety guarantee compared to \ac{vlm}-as-grounder approaches, as shown in \cref{ap:safety}.

The computational requirements of \cref{alg:replan} are divided into three components.
The \ac{vlm} is queried once per fluent, so this stage is linear in $|\fluentSet|$ (number of \ac{vlm} calls), and in practice dominates wall-clock time because \ac{vlm} calls are expensive. These independent queries can be batched or restricted to task-relevant fluents.
The \ac{mlss} algorithm has a worst-case exponential time complexity in the number of fluents, but as shown in \cref{thm:efficient}, this is manageable when the number of constraint-violating states is limited.
Finally, \ac{cp} itself is worst-case exponentially hard, meaning \ac{sc} inherits the conformant planners' theoretical complexity.

\vspace{-18pt}

\section{Use Case for Robust Visual Planning with a Simulated Household Robot}

\noindent\textbf{Domain.}
We evaluate \ac{sc} on ViPlan-HH \citep{merlerViPlanBenchmarkVisual2025}, a household robotics benchmark built on iGibson \citep{li2022igibson}.
We enhance the benchmark by removing any privileged information about hidden objects, making this a true partially observable domain\footnote{All code is available at \url{https://github.com/CLAIR-LAB-TECHNION/ViPlanPO/}}.
The robot is a mobile manipulator with an RGB camera and one arm. Tasks include sorting books on shelves, cleaning out drawers, locking doors, packing and unpacking groceries, and other household rearrangement tasks.
Tasks are categorized into three difficulty levels (simple, medium, hard) based on the task horizon, meaning that harder tasks require more actions to complete, but are not necessarily more complex.
See \cref{ap:domain} for details on the domain and tasks.
The domain is specified in \ac{pddl} with movable and fixed objects, relations such as \texttt{ontop} and \texttt{inside}, and high-level actions such as \texttt{navigate-to}, \texttt{grasp}, and \texttt{place-on}.
Motion actions use probabilistic executors and may fail because of kinematics, collisions, or similar constraints.
Observations are egocentric RGB images, so the robot sees only part of the home and must reason about hidden objects while replanning from new views.
Sample \ac{pddl} files appear in \cref{ap:pddl}.

\noindent\textbf{Setup.}
We compare three planning-loop paradigms from \cref{fig:methods:theirs,fig:methods:ours}, namely \ac{vlm}-as-planner (VLM-P), \ac{vlm}-as-grounder (VLM-G), and \ac{sc}.
All methods receive the current image and \acs{pddl}-derived context.
We use GPT-4.1 \citep{openaiGPT4VisionSystemCard2023} through the OpenAI API for all methods.
The planner baseline returns text actions, while \ac{vlm}-as-grounder and \ac{sc} use next-token probabilities from fluent queries.
Implementation details, baselines, and prompts are in \cref{ap:use-case,ap:baselines,ap:prompts}.

\begin{table}[h]
  \centering
  \caption{ViPlan-HH results. \ac{vlm}-P is \ac{vlm}-as-planner, \ac{vlm}-G is \ac{vlm}-as-grounder, and \ac{sc} is our method. Success and valid-first-plan rates are percentages. Lower action and planner-call counts are better.}
  \label{tab:results}
  \resizebox{\columnwidth}{!}{
\begin{tabular}{lccccccccc}
\toprule
 & \multicolumn{3}{c}{Simple} & \multicolumn{3}{c}{Medium} & \multicolumn{3}{c}{Hard} \\
 & VLM-P & VLM-G & \shortstack{RoVLaP\\(ours)} & VLM-P & VLM-G & \shortstack{RoVLaP\\(ours)} & VLM-P & VLM-G & \shortstack{RoVLaP\\(ours)} \\
\midrule
Success (\%) & 20.0 & 40.0 & \textbf{86.7} & 7.1 & 21.4 & \textbf{57.1} & 33.3 & 0.0 & \textbf{66.7} \\
Valid first plan (\%) & 46.7 & 40.0 & \textbf{66.7} & 0.0 & \textbf{50.0} & \textbf{50.0} & 11.1 & \textbf{22.2} & \textbf{22.2} \\
\# Actions & 7.3 & 7.1 & \textbf{5.9} & 15.0 & 13.8 & \textbf{11.2} & 17.4 & 20.0 & \textbf{13.2} \\
\# Planner calls & 7.3 & 4.4 & \textbf{2.3} & 15.0 & 6.4 & \textbf{1.5} & 17.4 & 10.3 & \textbf{1.9} \\
\bottomrule
\end{tabular}
}
\vspace{4pt}

  \vspace{-8pt}
\end{table}

\noindent\textbf{Quantitative Results.}
\Cref{tab:results} reports task execution success, valid-first-plan rates, and mean action counts and planner calls per trial.
\Ac{sc} achieves the highest success rate in every split.
Relative to VLM-G, success improves by 116.8\% on simple tasks and 166.8\% on medium tasks. On hard tasks, VLM-G solves no instances while \ac{sc} solves 66.7\%.
Relative to VLM-P, success improves by 333.5\%, 704.2\%, and 100.3\% on simple, medium, and hard tasks, respectively.
Compared to both baselines, \ac{sc} fully solves some task families for which neither baseline solves any instance.
These performance gains also come with reduced planning overhead.
\Ac{sc} uses between 1.9 and 10.0 times fewer planner calls than the baselines, and executes between 1.2 and 6.8 fewer actions on average.
The 66.7\% valid-first-plan rate on simple tasks indicates that considering multiple initial-state hypotheses is often sufficient to generate a valid plan without replanning.
This advantage for \ac{sc}, however, is only visible in the simplest tasks.

\noindent\textbf{Qualitative Behavior.}
The same mechanism explains the observed failures.
In \emph{cleaning out drawers}, the VLM-P baseline assumes the hidden bowl is reachable and plans a direct grasp, which fails when the bowl is inside a closed cabinet.
\Ac{sc} assigns probability to the hidden-object state and opens the cabinet first, so the plan succeeds whether or not the bowl is hidden.
In \emph{sorting books}, the image can make objects near the gripper appear held, and a deterministic grounder may conclude that the robot already holds the hardback.
\Ac{sc} keeps both holding hypotheses and selects a plan that is feasible under either one, avoiding the brittle dependence on a single perceptual judgment.

\vspace{-10pt}

\section{Conclusion}
To enable robotic agents to operate effectively in real-world, partially observable environments characterized by limited and uncertain perceptual information, we propose a framework for robust planning under perceptual uncertainty. The approach employs \acfp{vlm} to derive a probabilistic, factored representation of the current state, which is used to update the agent's belief. This belief representation supports the synthesis of robust task plans using off-the-shelf planners. As a natural extension of this work, we will incorporate active sensing into the planning process, enabling agents to reason explicitly about the informational value of sensing actions. In addition, we plan to embed the proposed methodology within embodied robotic systems and evaluate its effectiveness across a suite of complex task-and-motion planning (TAMP) domains.

\acks{
Beyond the experiments, we used generative AI tools in the following ways:
\begin{itemize}
    \item Gemini's \citep{GoogleGemini} deep research feature was used to verify novelty after the literature review.
    \item ChatGPT \citep{openaiChatGPT} was used to rephrase text and detect basic syntax and grammar errors.
    \item Codex \citep{chenEvaluatingLargeLanguage2021} was used for initial implementations of some components in the experimental code.
\end{itemize}
}

\bibliography{references}

\appendix

\crefalias{section}{appendix}
\crefalias{subsection}{appendix}

\section{Probabilistic Grounding Calculation}\label{ap:probground}

The exact computation of the probabilistic grounding of fluent \fluent given observation \obs is:
\begin{enumerate}
    \item Extract the \ac{vlm} output probabilities for the tokens ``true'', ``false'', and ``null''. Denote $\prob^{true}_{\obs,\fluent} = \vlmFn(\obs, \prompt_{\fluent})[true]$, $\prob^{false}_{\obs,\fluent} = \vlmFn(\obs, \prompt_{\fluent})[false]$, $\prob^{null}_{\obs,\fluent} = \vlmFn(\obs, \prompt_{\fluent})[null]$.
    \item If $\prob^{null}_{\obs,\fluent} > \prob^{true}_{\obs,\fluent}$ and $\prob^{null}_{\obs,\fluent} > \prob^{false}_{\obs,\fluent}$, set $\prob_{\obs,\fluent} = 0.5$.
    \item Otherwise, normalize the true-false probabilities: $\prob_{\obs,\fluent} = \frac{\prob^{true}_{\obs,\fluent}}{\prob^{true}_{\obs,\fluent} + \prob^{false}_{\obs,\fluent}}$.
\end{enumerate}

\section{Proof of \texorpdfstring{\cref{thm:mlss}}{Theorem~\ref{thm:mlss}}}\label{ap:proof-mlss}

\begin{definition}[Bit diffset]
    Let $\state,\state' \in 2^\fluentSet$. The \emph{bit diffset} $D(\state, \state')$ is defined as the set of fluents that differ between \state and $\state'$, i.e., 
    $$
    D(\state, \state') = \{\fluent \in \fluentSet | \state(\fluent) \neq \state'( \fluent)\}
    $$
\end{definition}

\begin{definition}[Bit-flip operator]\label{def:bit-flip}
    Let $\state \in 2^\fluentSet$ and $\fluent \in \fluentSet$. The \emph{bit-flip operator} $\text{Flip}(\state, \fluent)$ returns a new state $\state' \in 2^\fluentSet$ such that:
    $$
    \forall \fluent' \in \fluentSet \quad \state'(\fluent') = 
    \begin{cases}
        1 - \state(\fluent) & \text{if } \fluent' = \fluent \\
        \state(\fluent') & \text{otherwise}
    \end{cases}
    $$
\end{definition}

\begin{definition}[Canonical path]\label{def:canon-path}
    Let $\state,\state' \in 2^\fluentSet$. Let $D(\state, \state') = \{\fluent_1, ..., \fluent_n\}$ where the fluents are numbered according to a fixed total ordering over \fluentSet.
    The \emph{canonical path} from \state to $\state'$ is the path of states $\state = \state_0 \rightarrow \state_1 \rightarrow ... \rightarrow \state_n = \state'$ achieved by flipping the bits of state $\state_0$ one by one in the order of the fluents.
    That is, for $i \in \{1, ..., n\}$, define $\state_i = \text{Flip}(\state_{i - 1}, f_i)$, from $\state_0 = \state$ leading to $\state_n = \state'$.
\end{definition}

\begin{lemma}\label{le:res-vis}
    In \cref{alg:mlss}, let \queueExtracted denote the set of states that have been extracted from the heap \queue thus far. Then at the moment of max-extraction from \queue, it holds that $\queue \cup \queueExtracted = \text{Visited}$.
\end{lemma}

\begin{proof}
Note that every insertion to the heap is accompanied by an insertion of the same state to the visited set.
All states that have ever been added to the heap are either still in the heap or have been extracted from it.
Since no state is ever removed from the visited set, the lemma holds.
\end{proof}

\begin{lemma}\label{le:path-mono}
    In \cref{alg:mlss}, the probabilities of the states on the canonical path from $\state_0$ to any state \state are monotonically non-increasing in the states' unconstrained belief values, i.e.:
    $$
    \unconstBelief(\state_0) \geq ... \geq \unconstBelief(\state_n) = \unconstBelief(\state)
    $$
\end{lemma}

\begin{proof}
    Assume by contradiction that there exists $0 < j \leq n$ such that $0 \leq \unconstBelief(\state_{j - 1}) < \unconstBelief(\state_j)$.
    Then all of the components that comprise the product in $\unconstBelief(\state_j)$ are non-zero.
    Thus, we can safely divide by
    $$
    \prod_{\fluent \in \fluentSet \setminus \{\fluent_j\}} \factoredBelief(\fluent)^{\state_{j - 1}(\fluent)} (1 - \factoredBelief(\fluent))^{1 - \state_{j - 1}(\fluent)}
    $$
    From here, we have:
    \begin{align*}
        & \unconstBelief(\state_{j - 1}) < \unconstBelief(\state_j) \\
        \overset{1}{\iff} & \prod_{\fluent \in \fluentSet} \factoredBelief(\fluent)^{\state_{j - 1}(\fluent)} (1 - \factoredBelief(\fluent))^{1 - \state_{j - 1}(\fluent)} < \prod_{\fluent \in \fluentSet} \factoredBelief(\fluent)^{\state_{j}(\fluent)} (1 - \factoredBelief(\fluent))^{1 - \state_{j}(\fluent)} \\
        \overset{2}{\iff} & \factoredBelief(\fluent_j)^{\state_{j - 1}(\fluent_j)} (1 - \factoredBelief(\fluent_j))^{1 - \state_{j - 1}(\fluent_j)} < \factoredBelief(\fluent_j)^{\state_{j}(\fluent_j)} (1 - \factoredBelief(\fluent_j))^{1 - \state_{j}(\fluent_j)} \\
        \overset{3}{\iff} & \factoredBelief(\fluent_j)^{\state_{j-1}(\fluent_j)} (1 - \factoredBelief(\fluent_j))^{1 - \state_{j-1}(\fluent_j)} < \factoredBelief(\fluent_j)^{1 - \state_{j-1}(\fluent_j)} (1 - \factoredBelief(\fluent_j))^{\state_{j-1}(\fluent_j)}
    \end{align*}
    Transition 1 follows the definition of \unconstBelief in \cref{eq:calc-unconst-belief}. 
    Transition 2 is a division by all components in the product except for the one corresponding to $\fluent_j$, which are non-zero, which are equal since $\state_{j-1}(\fluent) = \state_j(\fluent)$ for all $\fluent \neq \fluent_j$ by the definition of the canonical path in \cref{def:canon-path}.
    Transition 3 follows from the definition of the bit-flip operation in \cref{def:bit-flip}, which states that $\state_j(\fluent_j) = 1 - \state_{j - 1}(\fluent_j)$.

    If $\state_{j-1}(\fluent_j) = 1$, then the above simplifies to
    \begin{align*}
    & \factoredBelief(\fluent_j)^1 (1 - \factoredBelief(\fluent_j))^0 < \factoredBelief(\fluent_j)^0 (1 - \factoredBelief(\fluent_j))^1 \\
    & \Rightarrow \factoredBelief(\fluent_j) < 0.5
    \end{align*}
    If $\state_{j-1}(\fluent_j) = 0$, then the above simplifies to
    \begin{align*}
    & \factoredBelief(\fluent_j)^0 (1 - \factoredBelief(\fluent_j))^1 < \factoredBelief(\fluent_j)^1 (1 - \factoredBelief(\fluent_j))^0 \\
    & \Rightarrow \factoredBelief(\fluent_j) > 0.5
    \end{align*}
    Since $\state_{j-1}(\fluent_j) = \state_0(\fluent_j)$, then both cases contradict the definition of $\state_0$, which for all fluents is defined as $\state_0(\fluent) = \indicator[\factoredBelief(\fluent) > 0.5]$.
\end{proof}

\begin{lemma}\label{le:order}
    In \cref{alg:mlss}, when a state $\state \in 2^\fluentSet$ is extracted from the max-heap \queue, then for all $\state' \in 2^\fluentSet$ not yet extracted from \queue, it holds that $\unconstBelief(\state) \geq \unconstBelief(\state')$.
\end{lemma}

\begin{proof}
Let $\queueExtracted_{\beliefSize} = \{\state_0, ..., \state_{\beliefSize - 1}\}$ be the set of the first $\beliefSize$ states extracted from \queue, numbered by the order in which they were extracted.
We must show that for all \beliefSize and for all states $\state \in  \queueExtracted_{\beliefSize}$, and for all $\state' \notin \queueExtracted_{\beliefSize}$, it holds that $\beliefState(\state) \geq \beliefState(\state')$.
We do this by induction on $\beliefSize$.

\paragraph{Basis:} When $\beliefSize = 1$, this is the first iteration where the only value in the heap is $\state_0$. By definition, for any $\state \neq \state_0$, it holds that $\unconstBelief(\state_0) \geq \unconstBelief(\state)$

\paragraph{Assumption:} Assume that for some $\beliefSize$ it holds that for all $s \in \queueExtracted_{\beliefSize}$ and $\state' \notin \queueExtracted_{\beliefSize}$ we have $\unconstBelief(\state) \geq \unconstBelief(\state')$.

\paragraph{Induction Step:} Let $\queueExtracted_{\beliefSize + 1} = \{\state_0, ..., \state_{\beliefSize}\}$.
By the induction assumption, $\queueExtracted_{\beliefSize} = \queueExtracted_{\beliefSize + 1} \setminus \{\state_{\beliefSize}\}$ contains the top $\beliefSize$ most probable states, i.e., for all $\state \in \queueExtracted_{\beliefSize}$ and for all $\state' \notin \queueExtracted_{\beliefSize}$, it holds that $\unconstBelief(\state) \geq \unconstBelief(\state')$.
It is left to show that for all $\state \notin \queueExtracted_{\beliefSize + 1}$, it holds that $\unconstBelief(\state_{\beliefSize}) \geq \unconstBelief(\state)$.

Assume by contradiction that there exists a state $\state \notin \queueExtracted_{\beliefSize + 1}$ such that $ \unconstBelief(\state) > \unconstBelief(\state_{\beliefSize})$.
Then $\state \notin \queue$ because otherwise it would have been extracted before $\state_{\beliefSize}$.
By \cref{le:res-vis}, $\state$ has not yet been visited.

On the canonical path from $\state_0$ to $\state_n = \state$, since $\state_0$ has been visited and $\state_n = \state$ has not yet been visited, there exists $0<j\leq n$ such that $\state_j$ has not yet been visited and $\state_{j-1}$ has.
\begin{itemize}
    \item Since $\state_{j-1}$ is in the visited set, then by \cref{le:res-vis}, it is either in the heap or has been extracted from the heap. It cannot have been extracted from the heap because then it would have been expanded, and $\state_j$ would have been visited. Therefore, $\state_{j-1}$ is in the heap.
    \item Since $\state_\beliefSize$ was extracted but $\state_{j-1}$ was not, it follows that $\unconstBelief(\state_\beliefSize) \geq \unconstBelief(\state_{j - 1})$.
    \item By \cref{le:path-mono}, $\unconstBelief(\state_{j - 1}) \geq \unconstBelief(\state_n) = \unconstBelief(\state)$.
\end{itemize}

Putting it all together, we get:
$$
\unconstBelief(\state_\beliefSize) \geq \unconstBelief(\state_{j - 1}) \geq \unconstBelief(\state) > \unconstBelief(\state_\beliefSize)
$$
which is a contradiction.
Thus, for all $\state \notin \queueExtracted_{\beliefSize + 1}$, it holds that $\unconstBelief(\state_{\beliefSize}) \geq \unconstBelief(\state)$.
\end{proof}

\textbf{MLSS Theorem:} In \cref{alg:mlss}, when a constraint-satisfying state $\state$ is extracted from the heap $\queue$, then for all $\state'\notin \beliefSet$ it holds that $\beliefState(\state) \geq \beliefState(\state')$, where \beliefState is the belief computed by \cref{eq:calc-belief}.

\begin{proof}
Let $\state$ be a constraint-satisfying state extracted from the heap at some iteration.
We must show that for all $\state' \notin \beliefSet$, it holds that $\beliefState(\state) \geq \beliefState(\state')$.

If $\state'$ violates the constraints, then trivially $\beliefState(\state') = 0 \leq \beliefState(\state)$. Otherwise, note that $\beliefSet$ is comprised of all constraint-satisfying states that were extracted from the heap before \state. Therefore, $\state'$ has not yet been extracted from the heap. By \cref{le:order}, we have:
$$
\unconstBelief(\state) \geq \unconstBelief(\state') \iff \frac{\unconstBelief(\state)}{\normalizer} \geq \frac{\unconstBelief(\state')}{\normalizer} \iff \beliefState(\state) \geq \beliefState(\state')
$$
Thus, the theorem holds for all $\state' \notin \beliefSet$.
\end{proof}

\section{Proof of \texorpdfstring{\cref{thm:efficient}}{Proposition~\ref{thm:efficient}}}

\subsection{Time Complexity.}

\begin{claim}
The time complexity of \cref{alg:mlss} is $O(|\fluentSet|\beliefSize\log(|\fluentSet|\beliefSize))$
\end{claim}

\begin{proof}
    \textbf{Pre-iteration cost.}
In lines \ref{alg:mlss:init-s0}-\ref{alg:mlss:init-p} of the algorithm, we linearly compute $\state_0$ and its probability, and initialize collections with at most one item.
This requires $O(|\fluentSet|)$ operations.

\textbf{Inner-iteration cost.}
The internal \textit{for} loop in line \ref{alg:mlss:for} consists of a bit-flip, a ``visited'' check, an insertion to the visited set, and an insertion to the heap. Using random access, the bit flip is $O(1)$. Checking membership in the visited set is $O(1)$. Inserting into the visited is $O(\log(M))$ where $M$ is the size of the set. The heap insertion is $O(\log(N))$ where $N$ is the size of the heap. Assuming none of the states are visited (worst case), then we require $O(|\fluentSet|(\log(N) + \log(M)))$ for all iterations. We must also acknowledge the time to compute the probability at each iteration. If naively implemented, this would take $O(|\fluentSet|)$, meaning throughout the iteration, this would take $O(|\fluentSet|^2)$. However, by updating the probability incrementally according to only the changed bit (divide by unflipped value and multiply by flipped value), this operation takes $O(1)$.

\textbf{Outer-iteration cost.}
The \textit{while} loop in line \ref{alg:mlss:while} consists of a max extraction from the heap, an insertion of this value to \beliefSet, and the internal \textit{for} loop. Since \beliefSet is at most size $\beliefSize$, the insertion can be implemented in $O(\log(\beliefSize))$. The max extraction from the heap requires $O(\log(M))$ where $M$ is the size of the heap.  The \textit{while} conditions are checkable in $O(1)$. The inner \textit{for} loop is $O(|\fluentSet|(\log(N) + \log(M)))$. Considering that at most $|\fluentSet|$ values are inserted into the visited set and heap at every iteration, then $M = N = |\fluentSet|\beliefSize$ at most. Overall, each iteration in the \textit{while} loop requires:
\begin{align*}
    & O(\log(\beliefSize)) + O(\log(M)) + O(|\fluentSet|(\log(N) + \log(M))) \\
    = & O(\log(\beliefSize)) + O(|\fluentSet|(\log(|\fluentSet|\beliefSize) + \log(|\fluentSet|\beliefSize))) \\
    = & O(|\fluentSet|\log(|\fluentSet|\beliefSize))
\end{align*}
There are at most $\beliefSize$ iterations, so the overall time complexity of the while loop is $O(|\fluentSet|\beliefSize\log(|\fluentSet|\beliefSize))$.

Since the pre-iteration costs are negligible compared to the iteration cost, the final time upper-bound complexity stands at $T(\fluentSet, \beliefSize) = O(|\fluentSet|\beliefSize\log(|\fluentSet|\beliefSize))$.
\end{proof}

\subsection{Space Complexity}

\begin{claim}
    The space complexity of \cref{alg:mlss} is $O(|\fluentSet|^2\beliefSize)$ bits.
\end{claim}
\begin{proof}
    As seen in the proof of time complexity, we store the following values at any one time:
\begin{itemize}
    \item The visited set, containing $O(|\fluentSet|\beliefSize)$ states. 
    \item The heap, containing $O(|\fluentSet|\beliefSize)$ states.
\end{itemize}
Each state can be represented using a bitmask of $|\fluentSet|$ bits. Therefore $S(\fluentSet, \beliefSize) = O(|\fluentSet|^2\beliefSize)$ bits.
\end{proof}

\section{Proof of Finite Step Convergence}

In this section, we prove \cref{thm:fin-step} via \cref{thm:calib}.

\subsection{Proof of \texorpdfstring{\cref{thm:calib}}{Theorem~\ref{thm:calib}}}\label{ap:proof:calib}

\textbf{Theorem Statement.} Let $\probThresh \in (0,1]$, let \constraintSet be a set of constraints over \fluentSet, and let $\state \in 2^\fluentSet$ satisfy \constraintSet. For every $\factoredBelief \in [0,1]^\fluentSet$, if $\calibError(\factoredBelief, \state) < \probThresh$, then $\state \in MLSS(\factoredBelief, \probThresh, \constraintSet)$.

\begin{proof}
Let $\state \in 2^\fluentSet$ satisfy \constraintSet and let $\factoredBelief \in [0,1]^\fluentSet$. For simplicity, we denote $\calibError = \calibError(\factoredBelief,\state)$.

Let $\error_\fluent = |\beliefState_\fluent - \state(\fluent)|$ be a single fluent belief error, and
let $\weightProb(\state) = \prod_{\fluent \in \fluentSet} \beliefState_\fluent^{\state(\fluent)} (1 - \beliefState_\fluent)^{1 - \state(\fluent)}$ be the unnormalized probability weight of state \state.

If $\state(\fluent) = 1$ then
\begin{align*}
    & \error_\fluent = |\beliefState_\fluent - \state(\fluent)| = 1 - \beliefState_\fluent \\
    & \Rightarrow \beliefState_\fluent =  1 - \error_\fluent
\end{align*}
and if $\state(\fluent) = 0$ then
\begin{align*}
    & \error_\fluent = |\beliefState_\fluent - \state(\fluent)| = \beliefState_\fluent \\
    & \Rightarrow 1- \beliefState_\fluent =  1 - \error_\fluent
\end{align*}
Then by the Weierstrass product inequality, the unnormalized weight assigned to $\state$ is:
$$
\weightProb(\state) = \prod_{\fluent \in \fluentSet} (1 - \error_\fluent) \geq 1 - \sum_{\fluent \in \fluentSet} \error_\fluent = 1 - \calibError
$$

\beliefState is normalized by a normalizing factor $0 < \normalizer \leq 1$. Thus:
$$
\beliefState(\state) = \frac{\weightProb(\state)}{\normalizer} \geq \weightProb(\state) \geq 1 - \calibError
$$

Assume by contradiction that $\state \notin MLSS(\factoredBelief, \probThresh, \constraintSet)$. Then
$$
\beliefState(\state) \leq \sum_{\state' \notin MLSS(\factoredBelief, \probThresh, \constraintSet)}\beliefState(\state') \leq 1 - \probThresh.
$$

From the previous inequality:
\begin{align*}
    & 1 - \probThresh \geq \beliefState(\state) \geq 1 - \calibError \\
    \Rightarrow & \calibError \geq \probThresh
\end{align*}

This contradicts the theorem's assumption that $\calibError < \probThresh$, and thus $\state$ must be in the \ac{mlss}.

\end{proof}

\subsection{Proof of \texorpdfstring{\cref{thm:fin-step}}{Corollary~\ref{thm:fin-step}}}\label{ap:proof:fin-step}

\textbf{Theorem Statement:} Let $\minVisRate > 0$ be the minimum visibility rate. If the \ac{vlm} is weakly calibrated with error margin $\minPredImp$, then there exists $\finiteStep$ such that a conformant plan for the \ac{mlss} is satisficing from the true current state.

\begin{proof}
Let $\logObs_{\fluent,t}$ be the log odds of the correct assignment of \fluent at time $t$ of the evidence provided by the \ac{vlm}. Since the \ac{vlm} is weakly calibrated, then:

\begin{align*}
& \prob_{\text{correct}, t} - (1 - \prob_{\text{correct}}, t) \geq \prob_{\text{correct}, t} - \prob_{\text{incorrect}, t} \geq \minPredImp \\
\Rightarrow & \prob_{\text{correct}, t} \geq \frac{\minPredImp}{2} + 0.5
\end{align*}

This implies that $1 - \prob_{\text{correct}, t} \leq 0.5 - \frac{\minPredImp}{2}$. Thus:

\begin{align*}
    & \frac{\prob_{\text{correct}, t}}{1 - \prob_{\text{correct}, t}} \geq \frac{0.5 + \frac{\minPredImp}{2}}{0.5 - \frac{\minPredImp}{2}} \\
    \Rightarrow & \log\left(\frac{\prob_{\text{correct}, t}}{1 - \prob_{\text{correct}, t}}\right) \geq \log\left(\frac{0.5 + \frac{\minPredImp}{2}}{0.5 - \frac{\minPredImp}{2}}\right) \\
    \Rightarrow & \logObs_{\fluent,t} \geq \log\left(\frac{0.5 + \frac{\minPredImp}{2}}{0.5 - \frac{\minPredImp}{2}}\right) > 0
\end{align*}

Then \logOdds is bounded by a strictly positive constant. Denote $C = \log\left(\frac{0.5 + \frac{\minPredImp}{2}}{0.5 - \frac{\minPredImp}{2}}\right)$.

Let $\logOdds_{\fluent,t}$ be the log odds of the correct assignment of \fluent at time $t$ according to factored belief $\beliefState_{\fluentSet,t}$. The logarithmic pooling update rule is $\logOdds_{\fluent,t} = \logOdds_{\fluent,t-1} + \logObs_{\fluent,t}$. Assume W.L.O.G. that the initial belief $\beliefState_{\fluentSet,0}$ is uniform (because in any case, it is constant). Since we only update fluents that are visible, the log odds value at time $\finiteStep$ is:
\begin{align*}
\logOdds_{\fluent,\finiteStep} & = \log\left(\frac{\beliefState_{\fluent,0}}{1-\beliefState_{\fluent,0}}\right) + \sum_{t=1}^\finiteStep \mathbb{1}_\fluent(t)\cdot\logObs_{\fluent,t} \\
& = \sum_{t=1}^\finiteStep \mathbb{1}_\fluent(t)\cdot\logObs_{\fluent,t}
\end{align*}
where $\mathbb{1}_\fluent(t)$ is an indicator function for fluent \fluent being visible at timestep $t$. By the definition of \minVisRate, all fluents appear at least $\minVisRate\finiteStep$ times within $\finiteStep$ timesteps. Thus:
$$
\logOdds_{\fluent,\finiteStep} \geq \minVisRate\finiteStep\cdot C
$$

Let $\calibError_t$ be the cumulative factored belief error at time $t$, and let $\error_\fluent = |\beliefState_\fluent - \state^*(\fluent)| = 1 - \prob_{\text{correct}}$ where $\state^*$ is the true current state. Then:

\begin{align*}
\logOdds_{\fluent, t} & = \log\left(\frac{\prob_{\text{correct}}}{1 - \prob_{\text{correct}}}\right) \\
\logOdds_{\fluent, t} & = \log\left(\frac{1 - \error_{\fluent, t}}{\error_{\fluent, t}}\right) \\
e^{\logOdds_{\fluent, t}} & = \frac{1 - \error_{\fluent, t}}{\error_{\fluent, t}} \\
\error_{\fluent, t} & = \frac{1}{1 + e^{\logOdds_{\fluent, t}}} \leq e^{-\logOdds_{\fluent, t}} \leq e^{-\minVisRate\finiteStep\cdot C}
\end{align*}

We want the cumulative factored belief error to be less than \probThresh. Note that:
\begin{align*}
    \calibError_t = \sum_{\fluent \in \fluentSet} \error_{\fluent,t} \leq \sum_{\fluent \in \fluentSet} e^{-\minVisRate\finiteStep\cdot C} = |\fluentSet|\cdot e^{-\minVisRate\finiteStep\cdot C}
\end{align*}
Then it is enough to find \finiteStep such that
\begin{align*}
& |\fluentSet|\cdot e^{-\minVisRate\finiteStep\cdot C} < \probThresh \\
\Leftrightarrow & -\minVisRate\finiteStep\cdot C < \log\left(\frac{\probThresh}{|\fluentSet|}\right) \\
\Leftrightarrow & \finiteStep > \frac{\log(|\fluentSet|)  - \log(\probThresh)}{\minVisRate \cdot C}
\end{align*}

Since all values $|\fluentSet|,\probThresh,C,$ and $\minVisRate$ are all constant, there exists a finite \finiteStep such that this inequality always holds. Therefore, after the $\finiteStep$th belief update, $\state^*$ is in the \ac{mlss}, and so a conformant plan for all states in the \ac{mlss} is also a satisficing plan from $\state^*$.
\end{proof}

\section{Proof of \texorpdfstring{\cref{thm:safety}}{Theorem~\ref{thm:safety}}}\label{ap:safety}

An unsafe action is one for which the planner did not verify the preconditions.
As inspired by \citet{shaniReplanningDomainsPartial2011}, we define an action as safe if its preconditions hold for all states in the \ac{mlss}.
With the replanning strategy in \cref{alg:replan}, we can provide a safety guarantee compared to a \ac{vlm}-as-grounder policy \citep{merlerViPlanBenchmarkVisual2025}, denoted $\policy_{det}$, that always plans from the most probable state.

\begin{theorem}\label{thm:safety}
    Let $\text{unsafe}(\policy)$ denote the event that policy \policy executes an unsafe action. Then,
    $$
    \Pr(\text{unsafe}(\policy_{det})) \geq \Pr(\text{unsafe}(\policy_{cpp})) + (\probThresh - \beliefState(\state_{max})) \ \ s.t. \ \ \state_{max} \in \argmax_{\state} \beliefState(\state)
    $$
\end{theorem}

\begin{proof}
For belief $\beliefState$, $\policy_{det}$ generates a plan that is verified to be valid for the most likely state, denoted $\state_{max}$. Thus,

$$
\Pr(\text{unsafe}(\policy_{det})) = 1 - \beliefState(\state_{max})
$$

Policy $\policy_{cpp}$ generates a plan that is verified for all states in the \ac{mlss}, denoted $\beliefState_\probThresh$.
By the definition of $\beliefState_\probThresh$:

\begin{align*}
    & \sum_{\state \in \beliefState_\probThresh} \beliefState(\state) \geq \probThresh \\
    \Rightarrow & \Pr(\text{unsafe}(\policy_{cpp})) = 1 - \sum_{\state \in \beliefState_\probThresh} \beliefState(\state) \leq 1 - \probThresh
\end{align*}

From here:
\begin{align*}
    & \Pr(\text{unsafe}(\policy_{det})) \\
    = & 1 - \beliefState(\state_{max}) \\
    = & (1 - \probThresh) + (\probThresh - \beliefState(\state_{max})) \\
    \geq & \Pr(\text{unsafe}(\policy_{cpp})) + (\probThresh - \beliefState(\state_{max})) \\
    = & \Pr(\text{unsafe}(\policy_{cpp})) + (\probThresh - \max_{\state \in \stateSpace} \beliefState(\state))
\end{align*}
\end{proof}

\cref{thm:safety} states that our \ac{cpp} replanning policy is at least as safe as the VLM-as-grounder policy as long as the planning threshold is greater than the probability of the most likely state. We see that the safety advantage of our method grows exponentially with the entropy of the \ac{vlm}-induced belief. Although $\probThresh - \beliefState(\state_{max})$ can be negative if $\beliefState(\state_{max}) > \probThresh$, in practice \cref{alg:mlss} will select only this maximizer, making the resulting plan equivalent to that of the VLM-as-grounder policy.

\section{ViPlan-HH Domain and Evaluated Tasks}\label{ap:domain}

ViPlan-HH \citep{merlerViPlanBenchmarkVisual2025} is a simulated household robotics domain built on iGibson \citep{li2022igibson}. The agent is a mobile manipulator with an egocentric RGB camera and a single gripper. Each problem specifies a household scene, a set of typed objects, and a goal condition. The robot must navigate to relevant objects, manipulate movable items, and interact with containers such as drawers, cabinets, shelves, doors, and boxes.

\paragraph{Symbolic domain.}
The high-level model is represented in \ac{pddl}. Objects are typed as general objects, movable objects, containers, sliceable objects, and slicers. The main predicates describe reachability, grasp state, container state, object placement, containment, adjacency, and slicing, using predicates such as \texttt{reachable}, \texttt{holding}, \texttt{open}, \texttt{ontop}, \texttt{inside}, \texttt{nextto}, and \texttt{sliced}. The action set contains \texttt{navigate-to}, \texttt{grasp}, \texttt{place-on}, \texttt{place-next-to}, \texttt{place-inside}, \texttt{open-container}, \texttt{close-container}, and \texttt{slice}. Sample domain and problem files are shown in \cref{ap:pddl}.

The \ac{pddl} state is a high-level abstraction of the simulator state. Navigation makes a target object reachable and normally makes other objects unreachable, reflecting the robot's local viewpoint and manipulation range. Opening a reachable container makes its contents reachable, while closing a container hides them again. Placement actions require holding the moved object and reaching the target support or container. This abstraction captures the task-relevant structure while delegating low-level geometric execution to the simulator interface.

\paragraph{Observations and partial observability.}
At each step the agent observes an RGB image from the robot's current viewpoint. Objects outside the field of view, inside closed containers, or occluded by scene geometry are not directly observable. The original ViPlan-HH implementation supplements some non-visible predicates with privileged simulator state. We remove this privileged information in our experiments. As a result, hidden objects and unobserved relations must be handled through belief rather than by direct access to ground truth, making the setting a true partially observable visual planning problem.

\paragraph{Evaluated task families.}
The quantitative experiments use the ViPlan-HH difficulty splits, but only the task families for which the conformant-planning backend could be evaluated reliably. The evaluated task families are listed in \cref{tab:domain-tasks}. Task identifiers match the benchmark task names used by the code.

\begin{table}[h]
\centering
\caption{ViPlan-HH task families used in the quantitative evaluation.}
\label{tab:domain-tasks}
\begin{tabular}{ll}
\toprule
Difficulty & Evaluated task families \\
\midrule
Simple & \texttt{sorting\_books}, \texttt{cleaning\_out\_drawers}, \texttt{locking\_every\_door} \\
Medium & \texttt{packing\_food\_for\_work}, \texttt{sorting\_books}, \texttt{sorting\_groceries} \\
Hard & \texttt{organizing\_boxes\_in\_garage}, \texttt{putting\_away\_toys} \\
\bottomrule
\end{tabular}
\end{table}

The remaining ViPlan-HH task families were not included in the quantitative table because they exposed implementation-level nontermination in the Unified Planning CPOR backend used for conformant planning \citep{maliahComputingContingentPlan2022}. In these cases, the generated planning calls did not reliably return either a plan or a failure. The issue was in the planner stack rather than in the visual-planning policy. In particular, CPOR is invoked through the Python experiment harness but executes solver code in a separate .NET runtime, so Python-level timeouts did not consistently terminate the underlying solver process. This made the affected tasks unsuitable for automated, reproducible evaluation with the conformant-planning component.

We therefore report results on the largest subset of ViPlan-HH task families that could be run reliably under the same evaluation protocol for all compared methods. This filtering was based on backend evaluability, not on task outcome. The retained subset still spans all three official difficulty levels and includes tasks with hidden objects, container manipulation, object rearrangement, and long-horizon household goals. Excluding planner-backend hangs avoids conflating third-party solver nontermination with the visual planning questions studied in this paper.

\section{ViPlan-HH PDDL Files}\label{ap:pddl}

\subsection{Domain File}

\begin{lstlisting}
(define (domain igibson)
    (:requirements :strips :typing :negative-preconditions :conditional-effects :equality)

    (:types
        container movable - object
        sliceable slicer - movable
    )

    (:predicates

        ;; Agent predicates
        (reachable ?o - object)
        (holding ?m - movable)

        ;; Object attributes
        (open ?c - container)

        ;; Object relations
        (ontop ?o1 - object ?o2 - object) ;; no assumptions on the types of objects that can be on top or below others
        (inside ?o - object ?c - container) ;; only containers can contain objects
        (nextto ?o1 - object ?o2 - object) ;; no assumptions on the types of objects that can be next to each other

        ;; Specific object attributes
        (sliced ?s - sliceable) ;; (e.g. sliced tomato)
    )

    (:action grasp
        :parameters (?m - movable)
        :precondition (and
            (forall
                (?x - movable)
                (not (holding ?x))) ;; Agent must not be holding anything
            ;;(forall
            ;;    (?x - movable)
            ;;    (not (ontop ?x ?m))) ;; Can't grasp an object that has something on top of it
        )
        :effect (and
            (when
                (reachable ?m)
                (and
                    (holding ?m)
                    (forall
                        (?y - object)
                        (and
                            (not (ontop ?m ?y)) ;; If grasped object is on top of something, it is no longer on top of it
                            (not (nextto ?m ?y)))) ;; Same for next to
                )
            )
            (forall
                (?c - container)
                (when
                    (and
                        (reachable ?m)
                        (inside ?m ?c)
                    )
                    (not (inside ?m ?c)))) ;; If m was in a container, it's not anymore
        )
    )

    (:action place-on
        :parameters (?m - movable ?o2 - object)
        :precondition (and
            (reachable ?o2)
        )
        :effect (and
        (when
            (holding ?m)
            (and
                (ontop ?m ?o2)
                (not (holding ?m))
            )
        )
        )
    )

    (:action place-next-to
        :parameters (?m - movable ?o2 - object)
        :precondition (and
            (reachable ?o2)
        )
        :effect 
        (when
            (holding ?m)
            (and
                (nextto ?m ?o2)
                (not (holding ?m))
            )
        )
    )

    (:action place-inside
        :parameters (?m - movable ?c - container)
        :precondition (and
            (reachable ?c)
            (open ?c)
        )
        :effect 
        (when
            (holding ?m)
            (and
                (inside ?m ?c)
                (not (holding ?m))
            )
        )
    )

    (:action open-container
        :parameters (?c - container)
        :precondition (and
            (forall
                (?x - movable)
                (not (holding ?x))) ;; Agent must not be holding anything
        )
        :effect (and 
            (when
                (reachable ?c)
                (open ?c)
            )
            (forall
                (?o - object)
                (when
                    (and
                        (reachable ?c)
                        (inside ?o ?c)
                    )
                    (reachable ?o))) ;; All objects inside the container are reachable
        )
    )

    (:action close-container
        :parameters (?c - container)
        ; :precondition ()
        :effect (and
            (when
                (reachable ?c)
                (not (open ?c))
            )
            (forall
                (?o - object)
                (when 
                    (inside ?o ?c)
                    (not (reachable ?o))    
                ) ;; All objects inside the container are unreachable
            )
        )
    )

    (:action navigate-to
        :parameters (?o - object)
        :precondition (and
            ;; don't navigate-to things hidden in a closed container
            (forall
                (?c - container)
                (or
                    (not(inside ?o ?c))
                    (open ?c)
                )
            )
        )
        :effect (and
            (reachable ?o) ;; make target object reachable

            (forall
                (?x - object)
                (when
                    (not (= ?x ?o)) ;; condition
                    (not (reachable ?x)))) ;; effect

            ;; Also, if there exists a container which is ?o and that it's open,
            ;; set the objects inside as reachable
            (forall
                (?c - container ?x - object)
                (when
                    (and
                        (= ?c ?o)
                        (open ?c)
                        (inside ?x ?c)
                    )
                    (reachable ?x)))
        )
    )

    (:action slice
        :parameters (?o - sliceable ?s - slicer)
        :precondition (and
            (holding ?s)
            (reachable ?o)
            (not (sliced ?o))
        )
        :effect (and
            (sliced ?o)
        )
    )

)
\end{lstlisting}

\subsection{Example ``Sorting Books'' Problem File}
\begin{lstlisting}
(define (problem sorting_books_0)
    (:domain igibson)

    (:objects
     	hardback_1 - movable
    	table_1 - object
    	shelf_1 - object
    )
    
    (:init 
        (ontop hardback_1 table_1) 
    )
    
    (:goal 
        (and 
            (ontop hardback_1 shelf_1)
        ) 
    )
)
\end{lstlisting}

\subsection{Example ``Cleaning Out Drawers'' Problem File}
\begin{lstlisting}
(define (problem cleaning_out_drawers_0)
    (:domain igibson)

    (:objects
     	bowl_1 - movable
    	cabinet_1 - container
    	sink_1 - object
    )
    
    (:init 
        (inside bowl_1 cabinet_1) 
        (not (open cabinet_1))
    )
    
    (:goal 
        (and 
            (ontop bowl_1 sink_1) 
        )
    )
)
\end{lstlisting}

\section{Use Case: Full Details}\label{ap:use-case}

\subsection{Baselines} 
We compare three approaches for integrating \acp{vlm} into a planning loop, depicted in \cref{fig:methods:theirs,fig:methods:ours}. These are \ac{vlm}-as-planner, \ac{vlm}-as-grounder, and \ac{sc} (ours). The full implementation details are provided in \cref{ap:baselines}.

\subsection{Deeper Analysis}

\subsubsection{Scenario 1: Hidden Object}
\begin{sloppypar}
In the \emph{cleaning out drawers} task, the robot must place a bowl into the sink, but the bowl is initially out of view and located inside a closed cabinet (see \cref{fig:example:kitchen}).  
The VLM-P baseline consistently mispredicts the initial state, assuming the bowl is directly reachable. As a result, it generates the following plan: \verb|navigate-to(bowl_1)|, \verb|grasp(bowl_1)|, \verb|navigate-to(sink_1)|, \verb|place-on(bowl_1, sink_1)|.
This plan inevitably fails when the bowl is inside a cabinet, since grasping it requires opening the cabinet first.
\end{sloppypar}

\begin{sloppypar}
In contrast, \ac{sc} assigns non-zero probability to the bowl being hidden, producing the plan: \verb|navigate-to(cabinet_1)|, \verb|open-container(cabinet_1)|, \verb|navigate-to(bowl_1)|, \verb|grasp(bowl_1)|, \verb|navigate-to(sink_1)|, \verb|place-on(bowl_1, sink_1)|.
This plan is valid whether the bowl is actually inside the cabinet or not.
By choosing a robust plan that includes opening the cabinet, \ac{sc} avoids premature commitment and expands the set of possible successful initial states, reducing the need for replanning.
\end{sloppypar}

\subsubsection{Scenario 2: Misleading Observation}
\begin{sloppypar}
In the \emph{sorting books} task, the robot must place a hardback book onto a shelf. The initial camera image can be misleading: objects near the gripper may appear as if they are being held (\cref{fig:example:radio}).
Under these conditions, the \ac{vlm}-as-grounder baseline often incorrectly interprets the scene and concludes that the robot is already holding the hardback. This misinterpretation causes the robot to generate the following plan: \verb|navigate-to(shelf_1)|, \verb|place-on(hardback_1, shelf_1)|.
Because the robot is not actually holding the \texttt{hardback\_1}, this plan fails.
\end{sloppypar}

\Ac{sc}, however, maintains uncertainty over the \texttt{holding} predicate and generates a plan that is feasible under both hypotheses (holding vs. not holding): \verb|place-on(hardback_1, shelf_1)|, \verb|navigate-to(hardback_1)|, \verb|grasp(hardback_1)|, \verb|place-on(hardback_1, shelf_1)|. This plan ensures task success in either case, avoiding the brittle dependence on a single incorrect perceptual judgment.

\section{Experimental Baselines}\label{ap:baselines}

\paragraph{\acs*{vlm}-as-planner.} This baseline, denoted VLM-P, implements the ViLa planning architecture \citep{huLookYouLeap2023}. The \ac{vlm} outputs a task plan in a specific format directly. The first action in that plan is taken as the agent's next action. A list of all previously selected actions is added to the \ac{vlm}'s context to provide the system with some memory of past interactions with the environment.

\paragraph{\acs*{vlm}-as-grounder.}
This is implemented as in \citet{merlerViPlanBenchmarkVisual2025} as a variation of \ac{s3e} \citep{azranS3ESemanticSymbolic2025}. The \ac{vlm} outputs a grounded \acs{pddl} state, obtained by asking a series of yes-no questions corresponding to grounded fluents. The grounded state is passed to the Fast Downward planner \citep{helmertFastDownwardPlanning2006} to produce a task plan that is passed to the executor. The executor uses the VLM to monitor execution by checking preconditions and effects of the new observations. When these become inconsistent with the current action, replanning is triggered.

Importantly, the implementation in \citet{merlerViPlanBenchmarkVisual2025} bypasses partial observability by supplementing the grounded state with privileged information from the simulation in the form of ground-truth assignments for predicates containing objects that are not visible. To increase the fidelity and reliability of our approach, we do not assume access to this information.  

\paragraph{\acs*{sc} (Ours).}
As described above, and depicted in \cref{fig:methods:ours}, in our approach, we extract probabilities from the \ac{vlm} for each grounded predicate via the logits from the final layer of the network. We use these probabilities to update a per-predicate belief using logarithmic opinion pooling belief update \citep{neymanNoRegretLearningUnbounded2023}.
Using \cref{alg:mlss}, a subset of states whose probability is at least some threshold probability is selected for planning. The state selection process uses constraints from the Fast Downward invariant finder \citep{helmertFastDownwardPlanning2006} to filter out impossible states. We then find a conformant plan for the selected subset of states, if one exists, using the CPOR planner \citep{maliahComputingContingentPlan2022} with a Fast Downward internal planner. If no plan is found, we perform a binary search to find the largest threshold probability for which a plan exists. If a plan is found, it is executed, updating the belief at every step.

\section{\acs*{vlm} Prompts}\label{ap:prompts}

We provide the exact prompts used for all baselines in our experiments. The \ac{vlm}-as-planner pipeline is prompted to generate a plan, while \ac{sc} and the \ac{vlm}-as-grounder pipeline are prompted to answer questions about the state of the environment. The prompts are divided into a system prompt, which describes the task and the environment, and a user prompt, which contains the specific question or goal for the current episode.

For \ac{vlm}-as-planner, the system and user prompts are as follows:
\begin{lstlisting}
<system> 
You are an expert planning assistant. You will be given an image which represents the current state of the environment you are in, a natural language description of the goal that needs to be achieved and a set of actions that can be performed in the environment. 
Your task is to generate a plan that achieves the goal, in the form of a sequence of actions that need to be executed to reach the goal.
The format of your output should be a JSON object with the following structure:
```json
{
  "plan": [
    {
        "action": action_name,
        "parameters": ['parameter1', 'parameter2', ...]
    },
    ... other actions ...
    ]
}
```

You will also receive feedback of the previously taken actions, with a note showing if they failed or not. If an action failed, think about why that could be and then output a new plan accordingly.
</system>
<user>

## Description of the environment
The environment is a virtual household simulator, with objects and furniture which can be interacted with. Keep in mind that some objects might not be visible or immediately reachable, in which case you need to navigate to them first. If after navigating to an object it is still not reachable, you might need to open a container.
Visible objects related to the task are highlighted with bounding boxes and labeled. Objects that are not in bounding boxes are not relevant to the task, ignore them when you answer.


## Available actions

- Action: grasp  
  - Parameters:  
    1. a movable object  
  - Preconditions:  
    - The object is within reach.  
    - The agent is not holding anything.  
  - Effects:  
    - The agent picks up that object.  
    - It is no longer on top of or next to any other object.  
    - If it was inside a container, it leaves the container.  

- Action: place-on  
  - Parameters:  
    1. the movable object being held  
    2. another object to serve as support  
  - Preconditions:  
    - The agent is holding the first object.  
    - The support object is within reach.  
  - Effects:  
    - The held object is placed on top of the support object.  
    - The agent's hands become free.  

- Action: place-next-to  
  - Parameters:  
    1. the movable object being held  
    2. another object to stand beside  
  - Preconditions:  
    - The agent is holding the first object.  
    - The other object is within reach.  
  - Effects:  
    - The held object is positioned next to the other object.  
    - The agent's hands become free.  

- Action: place-inside  
  - Parameters:  
    1. the movable object being held  
    2. an open container  
  - Preconditions:  
    - The agent is holding the object.  
    - The container is open and within reach.  
  - Effects:  
    - The object is placed inside the container.  
    - The agent's hands become free.  

- Action: open-container  
  - Parameters:  
    1. a closed container  
  - Preconditions:  
    - The container is within reach.  
    - The agent is not holding anything.  
  - Effects:  
    - The container becomes open.  
    - All objects inside it become reachable.  

- Action: close-container  
  - Parameters:  
    1. an open container  
  - Preconditions:  
    - The container is within reach.  
  - Effects:  
    - The container becomes closed.  
    - All objects inside it become unreachable.  

- Action: navigate-to  
  - Parameters:  
    1. any target object  
  - Preconditions:  
    - The target object is currently out of reach and not hidden in a closed container.  
  - Effects:  
    - The target object becomes reachable.  
    - All other objects become out of reach.  
    - If the target is an open container, everything inside it also becomes reachable.  

## Goal
{goal_string}

## Previously taken actions
{previous_actions}

</user>
\end{lstlisting}

The goal string is a natural language description of the goal that needs to be achieved, taken from a template with placeholders for the relevant objects.
For example, \verb|reachable(red-book)| is a goal condition, then the goal string would contain the string ``the red-book needs to be reachable by the agent''.

The previously taken actions are given as a list of action names and their parameters, e.g., \verb|grasp(red-book)|, \verb|place-on(red-book, table)|, etc.

For \ac{sc} and \ac{vlm}-as-grounder, the system and user prompts are as follows:

\begin{lstlisting}
<system>
You are tasked with replying to a question about the given image. You will only refer to objects that are marked in red bounding box and ignore the other objects. You will be given a single question, and will need to answer it ONLY with "yes" to answer positively, "no" to answer negatively, or "unknown" if there is not enough information to tell. Do not write anything else besides your answer.

The environment is a virtual household simulator. Keep in mind that some objects might not be visible or immediately reachable, they might be out-of-sight in some container.
There is a robotic arm, which is the agent, that can hold objects. Visible objects related to the task are highlighted with red bounding boxes and labeled. Objects that are not in bounding boxes are not relevant to the task, ignore them when you answer.
</system>
<user>
{fluent_question}
</user>
\end{lstlisting}

The fluent question is a natural language question about the status of a fluent, taken from a template with placeholders for the relevant objects. For example, \verb|open(cabinet)| is a fluent, then the fluent question would contain the string ``Is the cabinet currently open?''.

\section{Acronyms}
\begin{acronym}

\acro{cwa}[CWA]{Closed World Assumption}
\acro{vlm}[VLM]{Vision-Language Model}
\acro{vla}[VLA]{Vision-Language-Action}
\acro{strips}[STRIPS]{Stanford Research Institute Problem Solver}
\acro{cp}[CP]{conformant planning}
\acro{cpp}[CPP]{conformant probabilistic planning}
\acro{s3e}[S3E]{Semantic Symbolic State Estimation}
\acro{mlss}[MLSS]{Most Likely Subset of States}
\acro{pddl}[PDDL]{Planning Domain Definition Language}
\acro{rvp}[RVP]{Robust Visual Planning}
\acro{sc}[RoVLaP]{Robust Vision-Language Planning}

\end{acronym}

\end{document}